\documentclass[conference]{IEEEtran}

\usepackage{graphicx}
\usepackage{hyperref}
\usepackage{amsfonts,amsmath,amssymb,bm,color}

\DeclareMathOperator*{\argmin}{arg\,min}
\DeclareMathOperator{\E}{\mathbb{E}}

\def\a{\bm a}

\def\d{\bm d}

\def\w{\bm w}
\def\x{\bm x}

\def\z{\bm z}

\newtheorem{theorem}{Theorem}
\newenvironment{proof}{\rm {\em Proof}:}{\hfill$\square$}

\begin{document}
\title{Misconduct in Post-Selections and Deep Learning}

\author{\IEEEauthorblockN{Juyang Weng}
\IEEEauthorblockA{\textit{Brain-Mind Institute and GENISAMA} \\
\textit{Okemos, MI 48864 USA}\\
juyang.weng@gmail.com}
}
%
%

\maketitle
\begin{abstract}
This is a theoretical paper on ``Deep Learning'' misconduct in particular and Post-Selection in general.   As far as the author knows, the first peer-reviewed papers on 
Deep Learning misconduct are \cite{WengPSUTS21,WengMisleadAIEE23,WengMisconduct23}.  Regardless of learning modes, e.g., supervised, reinforcement, adversarial, and evolutional, almost all machine learning methods (except for a few methods that train a sole system) are rooted in the same misconduct--- cheating and hiding---(1) cheating in the absence of a test and (2) hiding bad-looking data.  It was reasoned in  \cite{WengPSUTS21,WengMisleadAIEE23,WengMisconduct23} that authors must report at least the average error of all trained networks, good and bad, on the validation set (called general cross-validation in this paper).   Better, report also five percentage positions of ranked errors.  From the new analysis here, we can see that the hidden culprit is Post-Selection.  This is also true 
for Post-Selection on hand-tuned or searched hyperparameters, because they are random, depending on random observation data. Does cross-validation on data splits rescue Post-Selections from the Misconducts (1) and (2)?   The new result here says: No.    Specifically, this paper reveals that using cross-validation for data splits is insufficient to exonerate Post-Selections in machine learning.   In general, Post-Selections of 
statistical learners based on their errors on the validation set are statistically invalid.    
\end{abstract}


\section{Introduction}
\label{SE:intro}

The so-called ``Deep Learning'' is a scheme that trains multiple networks (or called models)
each of which starts from a different set of initial parameters, optionally processed further by error-backprop or value-backprop training.  The primary misconduct is to report only the luckiest few networks from Post-Selection---selecting a few luckiest networks from $n$ trained networks using a validation set (Post-Selection Using Validation Set, PSUVS) or a test set (Post-Selection Using Test Set, PSUTS)\cite{WengPSUTS21}.   Both PSUVS and PSUTS lack a test.  

Weng \cite{WengMisleadAIEE23} proposed a Nearest Neighbor With Threshold (NNWT) classifier that guarantees to reach a zero validation error, due to the Post-Selection step during training.   Weng \cite{WengMisconduct23} proposed a simpler version of NNWT, called Pure Guess Nearest Neighbor (PGNN) by dropping the threshold in NNWT.  PGNN also guarantees to reach a zero validation error, due to the Post-Selection step during training.  With a lack of tests, NNWT and PGNN should not generalize well, because they simply find the luckiest fit in the absence of a test.  In theory, they ``beat'' the errors of all well-known AI systems, such as AlphaGo, AlphaZero, AlphaFold, ChatGPT, and Bard since none of them have claimed a zero validation error. 

In the ACM's investigative report triggered by the author's complaint against the works of Turing Award 2018, the investigative team mentioned the future possibility of cross-validation.  Although none of the Turing Awarded works used
cross-validation, the report seems to imply that cross-validation might exonerate Post-Selections.

This paper scrutinizes 
the process of system development that contains Post-Selections and cross-validation for data splits.  In particular, we apply cross-validation {\em for data splits} at both the input data end and output data end.  In between we have Post-Selections.  We call such cross-validation {\em for data splits} traditional cross-validation. 

Weng \cite{WengPSUTS21,WengMisleadAIEE23,WengMisconduct23} extended the cross-validation for 
the luck of system parameters, including architecture hyperparameters and neuronal weights, called cross-validation {\em on
system parameters}.   In this paper, we called it {\em general} cross-validation. 

Normally, in the absence of Post-Selection, there is a wall between the data and models.  We provide a model set $M$ first and then collect data set $D$ next.   After a fixed model $m_i\in M$ is determined, the model $m_i$ is then exposed to a fit data set $F\subset D$ in the training stage.  The performance of the trained model $m_i(F)$ is validated using a disjoint validation set $V\subset D$, where $F$ and $V$ are disjoint.   The key point is the performance of 
$m_i(F)$ on a future test $T\subset D$ that is disjoint with $F$ and $V$.  In the absence of Post-Selection, the performance of $m_i(F)$ on $V$ is expected to be similar to the performance of $m_i(F)$ on $T$, if $V$ and $T$ satisfy the same distribution, although that are disjoint (i.e., no common data). 

Post-Selection breaks the wall.  The Post-Selection step is the second step of the training stage, where the first step is the fit step.   In the Post-Selection step of the training stage, the author 
trains $n$ models, $m_i(F)$, $i=1, 2, ... n$.   Among the $n$ fit models, top $m\ll n$ (e.g., $m=5$ and $n=10,000$) models are post-selected {\em after} the errors of all $n$ models on $V$ are known.  Typically, the Post-Selection pickes the luckiest $m_i(F)$ on $V$, from $i=1, 2, ... , n$.  Then, the author reports the luckiest error and hides all bad-looking errors. 

Is Post-Selection a valid statistical process?  

Post-selection is also controversial in statistics.  The proposals that used Post-Selection in 
statistical inference were called Post-Selection Inference PoSI \cite{Berk13} where Post-Selection is limited to linear models.  Namely, among the set of all linear models $M=\{ m_i(F_i) |i=1, 2, ... , n\}$ where $m_i(F_i)$ is a linear model that fits $F_i \subset F$, pick the one that is the luckiest on $V$.  Therefore, each different selection of subset $F_i \subset F$ results in a different model $m_i(F)$, although all the candidate models are linear.  

In the field of statistics, van der Laan et al. 2007  \cite{Laan07} proposed a so-called ``Super Learner'', whose goal is to find better learners from a set $M$ that contains any available types of models, from Random Forests to neural networks.   Through a Post-Selection step, the Super Learner gives a weight to each model in $M$, so that the cross-validated error is minimized using an exhaustive search.  
   
The Super Learner model contains cross-validation for input data splits but not for output data splits (e.g., see Fig.1 of \cite{Laan07}).  For simulation experiments, the authors \cite{Laan07} generated additional validation data for their {\em mathematical equation models} to validate the outputs from the Super Learner.   However, no such output validation was mentioned for their real-data simulations (HIV and diabetes).  This author does not think that a Super Learner can do well for HIV and diabetes data sets in a new future test.  Our Lost Luck Theorem below predicts that Super Learner performs poorly for real data.   

In general, unless  we have additional information that one or more
candidate models have abstraction, invariance, and transfer, like observations from mathematical equations, Post-Selection cannot transfer rare luck among many naive networks to a new future test.  Similarly, a rare luck in the last lottery draw will not be repeated in the future lottery draw.

This paper shows that the Super Learner is a misleading procedure, even after this author adds a nest 
cross-validation to it.   A lucky learner, generated by NNWT or PGNN, can ``sneak'' into the Super Learner and fool it.   The Super Learner would give a 100\% weight to the lucky learner, and a 0\% weight to all other bad-looking learners, and reports a cross-validated zero error.  However, the lucky learner Post-Selected by the Super Learner will do badly in a future test.  Adding nest cross-validation to Super Learner does not help either.   

This paper has the following novelties beyond the existing papers on Deep Learning misconduct \cite{WengPSUTS21,WengMisleadAIEE23,WengMisconduct23}:  

(1)  General cross-validation.   We analyze the effects from not only randomly initialized weights but also architecture hyperparameters that are typically manually tuned.  Often, optimal architecture hyperparameters were also searched for along a grid.  Our new analysis here has discovered that architecture hyperparameters must be uniquely generated in a closed form inside each learner.  Namely, any Post-Selection must report the distribution of multiple trained systems, at least their five percentage positions of ranked errors.   Specifically, picking the luckiest vector of architecture hyperparameters is statistically invalid since the luck does not transfer to a new future test.  For example, the five percentage positions of ranked errors from multiple tried values of the threshold in NNWT must be reported, if NNWT does not provide a sole threshold.  If Deep Learning reports average as the general cross-validation requires, the accuracy of Deep Learning should be uselessly bad, like what \cite{GaoBEAN21} showed for the MNIST data. 

(2)  Traditional cross-validation.  Next, we discuss cross-validated risk from $n$ candidate learners (not-nested) like Super Learner  \cite{Laan07}.   Because the cross-validation does not include all guessed learners that NNWT and PGNN tried but dropped, the cross-validation gives zero risk for NNWT and PGNN.  PGNN is much slower than NNWT to develop, but the Post-Selection like Super Learner does not care.  Because of the absence of a test, the cross-validated zero risk from the Super Learner does not say anything about the error in a future test.   This seems to be true with almost all Deep Learning papers.  Therefore, traditional cross-validation does not rescue Deep Learning in particular and Post-Selection in general. 

(3) Nested traditional cross-validation.  We further discuss nest cross-validation, one for input and one for output, beyond Super Learner  \cite{Laan07}.  However, NNWT and GPNN also produce a zero value for the cross-validated input error and cross-validated output error. 
This means that nest cross-validation cannot rescue deep learning either.  Therefore, the nested traditional cross-validation does not rescue Deep Learning in particular and Post-Selection in general. 

This theoretical paper does not present new experiments.   The author believes that even the best data set that money can buy is always task-specific and may suffer from the said misconduct in a fatal way.  Therefore, the 
misconduct protocol is primary and experimental performance data is secondary.

The remainder of the paper is organized as follows.  Post-Selection is reviewed in Section~\ref{SE:Post}.  
Section~\ref{SE:Proto} discusses Experimental Protocols. 
Section~\ref{SE:LL} provides new results about Post-Selection as the Lost-Luck Theorem. 
Section~\ref{SE:CVPost} presents new results for cross-validated Post-Selection. 
Section~\ref{SE:CVNPost} establishes new results for nest cross-validated Post-Selection. 
In Section~\ref{SE:social}, we discuss the implications of the results here to some major social science issues.
Section~\ref{SE:conclusions} provides concluding remarks. 

\section{Post-Selection Misconduct}
\label{SE:Post}
Currently, well-known Deep Learning networks include Convolution Neural Networks (CNNs), such as AlexNet \cite{Krizhevsky17}, AlphaGo Zero \cite{Silver17}, AlphaZero \cite{Silver18}, AlphaFold \cite{Senior20}, MuZero \cite{Schrittwieser20}, AlphaDev \cite{Mankowitz23}, those in IBM Debater \cite{Slonim21} among others \cite{Tuncal20,Alshammari21,Kusonkhum22}.  For open contests with AlphaGo \cite{Silver16}, this author alleged that humans did post-selections from multiple AlphaGo networks on the fly when test data were arriving from Lee Sedol or Ke Jie  \cite{WengMisleadAIEE23}.   More recent Deep Learning publications are in the author's misconduct reports submitted to {\em Nature} \cite{WengNatureReport21} and {\em Science} \cite{WengScienceReport21}, respectively.  Many other types of networks also belong to this misconduct category, such as LSTM \cite{Hochreiter97}, ELM \cite{GBHuang05,HuangGB06}, and evolutionary computations that use a validation set to select surviving agents from many simulated agents.  The misconduct is rooted in  Post-Selection from multiple guessed systems. 

Suppose that an available data set is in the form $D=\{(\x, \z) |\x \in X, \z \in Z\}$, where 
$X$ is the input space and $Z$ is the output space.   $D$ is partitioned into three mutually disjoint sets, a fit set $F$, a validation set $V$ (like a mock exam), and a test set $T$ (like a contest) so that $D$ is the union of three sets from the same distribution.
\begin{equation}
D=F\cup V \cup T.
\label{EQ:disjoint}
\end{equation}  
The trainer possesses the fit set $F$ and the validation set $V$.  He trains a learner to fit $F$ and then checks the error of the learner on $V$, but the trainer should not possess the test set $T$ since the test should be conducted by an independent contest agency.  Otherwise, $V$ and $T$ become equivalent, and PSUVS and PSUTS become the same. 

\begin{figure}[tb]
	\centering
	\includegraphics[width=1.0\linewidth]{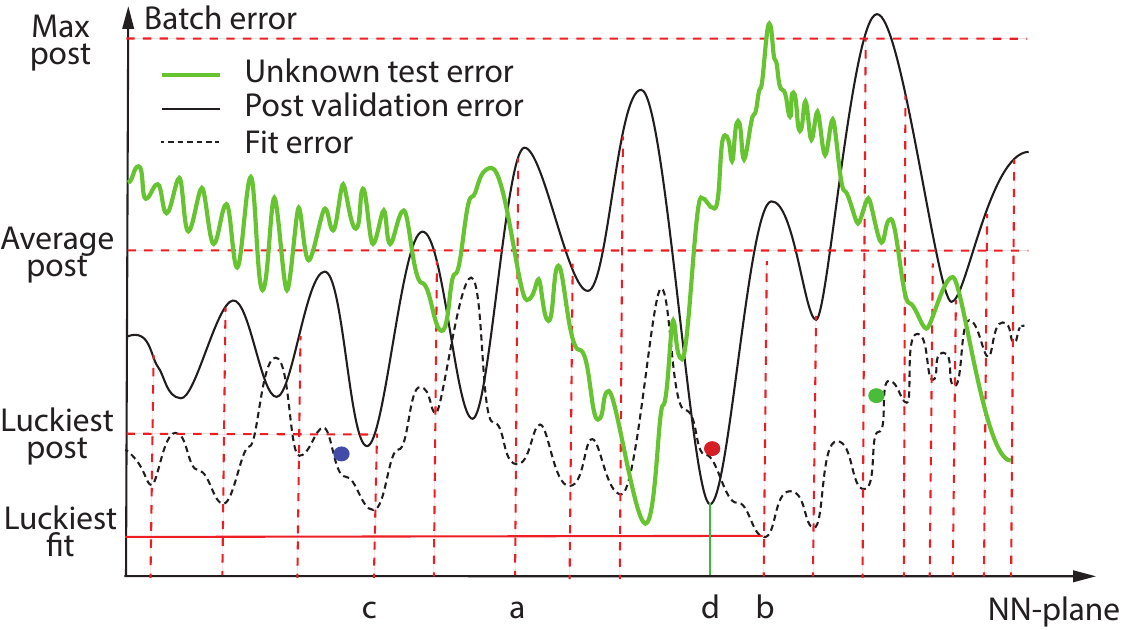}
	\caption{A 1D-terrain illustration for the fit error (dashed curve) from the fit data set, the post-selection validation error (solid thin curve) from the validation data set, and the unknown test error (thick green curve) from a future test set.  The green and blue NN-balls end at the max-post pit and the luckiest-post pit, respectively.  The red NN-ball will miss the lowest post error.  Only if $n$ is large ($n=3$ here), can the validation error from all $n$ network weight samples (i.e., cross-validation) better predict the expected error on the unknown test set.  The validation set and the test set have similar distributions here but are disjoint.  Figure modified from \cite{WengCLAIEE22}.}
	\label{FG:PSUTS-Fit-Error-1D}
\end{figure}

Fig.~\ref{FG:PSUTS-Fit-Error-1D} illustrates how any Post-Selection on a validation set $V$ does not help a future test on a test set $T$.  Each learner is represented by an NN-ball (Neural Network ball).  A gradient-based learning algorithm (e.g., error backprop or value backprop) makes all randomly initialized NN-balls roll down the hill on the terrain whose height is the error of NN-ball to fit $F$.  The errors on $F$, $V$, and $T$ are three different terrains. 

The alleged misconduct, Post-Selection, in so-called Deep Learning was originally raised in a YouTube video ``BMTalk 3D Episode 6: Did Turing Award Go to Fraud?'' posted on June 4, 2020.   The first peer-reviewed paper about the misconduct appeared in 2021 \cite{WengPSUTS-ICDL21}.  To 
explain why the Deep Learning scheme is misconduct, Weng 2023 \cite{WengMisleadAIEE23} proposed a simple classification method called
Nearest Neighbor With Threshold (NNWT) and proved mathematically that it reaches a zero error on any validation set under the misconduct protocol.   This NNWT was simplified by Weng 2023 \cite{WengMisconduct23} into   Pure-Guess Nearest Neighbor (PGNN) by dropping the threshold.  It was proven  \cite{WengMisconduct23} that PGNN also reaches a zero error on any validation set under the misconduct protocol.  Both NNWT and PGNN should be tested using a new test set that is not in the possession of the authors,  but almost all Deep-Learning publications have not published such a new test.  The misconduct consists of the following two:
\begin{enumerate}
\item Misconduct 1: Cheating in the absence of a test.   This is true for both PSUVS and PSUTS since the test set is often in the possession of the authors so the test set is used as a validation set and the validation errors were falsely reported as test errors.   
Namely, so-called ``test data'' in almost all Deep Learning publications are, as alleged by Weng, only training data, not a test set.   
\item Misconduct 2: Hiding bad-looking data.  Furthermore, Weng alleged that all bad-looking data, except that from the luckiest network, were hidden and not reported.  
\end{enumerate}  
Without the power of abstraction, NNWT and PGNN are expected to generalize badly, like a few Deep Learning papers that report average errors \cite{GaoBEAN21}.

Similarly, almost all performance data from evolutional methods are also misconduct because of the use of Post-Selection.  Namely, the performances of all individual networks in an evolutionary generation should be reported.  Furthermore, a reasonably disjoint test set must be used to evaluate the generalization of the luckiest network.

Furthermore, those methods that use a human programmer to tune parameters are also misconduct.  Such methods are numerous, from Neocognitron of Fukushima \cite{Fukushima80,Fukushima83}, to HMAX of Poggio et al. \cite{Serre07}, to almost all neural networks mentioned earlier, including Large Language Models (LLMs), Transformers, ChatGPT, and Bard.  This new information represents an advance from the author's earlier papers \cite{WengPSUTS21,WengMisleadAIEE23,WengMisconduct23}.  Namely, the authors must report at least the average and the five percentage positions, 0\%, 25\%, 50\%, 75\%, and 100\%, of the ranked errors of all trained networks.  Only such more informative data can give a reasonable description of the error distribution in a future test by even 
the luckiest network on the validation set.  In the absence of further information (e.g., about the abstraction and invariance inside the network), 
the expected error of the luckiest network on $V$ in a future test $T$ is the same as any other trained network, as I will prove below.  From Fig.~\ref{FG:PSUTS-Fit-Error-1D}, we can intuitively see why---the luck of the luckiest network on $V$ does not likely translate to a future test $T$. 

\section{Experimental Protocols}
\label{SE:Proto}
Next, let us discuss experimental protocols with the goal of the performance of the target system during future deployments.  The statistical principle that we focus on in this paper is a well-known protocol, called cross-validation (for data splits).  

The traditional cross-validation \cite{DudaHartStork,Laan03,Laan06} is for eliminating the luck in splitting all available data into a fit set and a validation set.   For example, if every validation point in the validation set is surrounded by 
some points in the fit set, the validation error is likely small.  For example, the nearest neighbor method, without much need for abstraction and invariance, will do reasonably well since for every point in the validation set---there is a nearby point in the fit set. 

Here, we extend the {\em traditional} cross-validation  \cite{DudaHartStork,Laan03,Laan06} to what is called 
{\em general} cross-validation.   To estimate the expected performance in a future test, we require 
resampling the parameter space that the programmer cannot directly compute in a closed form, such 
as the architecture hyperparameters and neuronal weights, by reporting the average of all resampled networks.    
 
 Currently, few Deep Learning networks report cross-validated errors.   One superficial reason is the high computational cost of training one network which typically involves GPUs.   However, this does not seem to be a valid reason, because almost all Deep Learning networks train $n\ge 2$ networks, e.g., $n=20$ in \cite{Graves16} and $n=10,000$ in \cite{Saggio21}.  The authors only report the error of the luckiest networks on a validation set that is in the procession of the authors.   The performances of the remaining $n-1$ bad-looking networks are hidden, and not reported.  Therefore, it is not the issue of computational cost, as such a cost has already been spent,  The key issue is what to report and whether bad-looking data are hidden. 
 
This paper will prove below that the luckiest network on the validation set should give approximately the average error in a future test.   This argument corresponds to the Lost-Luck theorem to be established in this paper. 

The Lost-Luck theorem is a generalization of the traditional cross-validation, by going beyond the splits of the data, but instead also including any randomly generated variables or manually tuned variables, as long as the set of tried system parameters $P=\{P_1, P_2, ... , P_n\}$ contains multiple elements ($n\ge 2$).  In splitting data into 
a fit set and a validation set, the set of system parameters contains $n$ elements, where each element 
corresponds to a different split between the fit set and the validation set.   
The  Lost-Luck theorem states that the average error of all trained systems, each from randomly generated parameters or manually tuned parameters, is a minimum variance estimate of the error in a future test.   

In our NNWT example, the threshold in NNWT is unlikely transferable to a future test, depending on the complexity of the data set or the AI problem.  In other words, there is no threshold for the distance between a query and its nearest neighbor so that the query and the nearest neighbor always share the same label.  For example, a shift in a pixel position may change the label of an object (e.g., a square vs. a rectangle).   

\section{Lost-Luck Theorem}
\label{SE:LL}

For simplicity, we consider only an episodic system $f: X \mapsto Z$, where $X$ consists of vectors 
of inputs and $Z$ consists of a finite set of class labels as desirable outputs.   For systems with state,  systems for vector outputs, and systems with internal representations, see \cite{WengMisconduct23} about how $Z$ is
extended.  The results here apply to state-based systems by changing $Z$ to $S\times X$, where $S$ is the set of states. 

\subsection{Cross-Validation on Data Splits}
Given a data set $D=(\d_1, \d_2, ..., \d_d)$, where $\d_i=(\x_i, l_i)$ consists of an input vector $\x_i \in X$, 
and $l_i \in Z$, $i=1, 2, ... , d$.  We need to divide the data set $D$ into two disjoint sets, a fit set $F$ and a validation set $V$, so that $D=F\cup V$.   We do not call $F$ training set because as we have seen in \cite{WengMisconduct23} $V$ was also used in the training stage due to Post-Selection.  Traditional cross-validation is a scheme for organizing and splitting the data set $D$.   Let us briefly overview four well-known types of cross-validation \cite{Laan03}.


$n$-fold cross-validation: The set $D$ is randomly divided into $n$ mutually exclusive and exhaustive sets 
$D_v, v=1, 2, ... , n$, of as nearly equal size as possible.  For each data split $D_v$, the corresponding fit set is defined as $F_v=D-D_v$ and the validation set is $V_v= D_v$.  Therefore, we get $n$ ways to split $D$ into a fit set and a validation set.   

 Leave-one-out cross-validation (LOOCV):   LOOCV is $n$-fold cross-validation when the fold size is 1, $n=d$.


Below, we consider $n$-fold cross-validation, with LOOCV as its special case.  The conclusions from these two types of cross-validation also hold for the other types of cross-validation, as long as the randomness in these four types is pseudo-random, which is typically true with computer-generated pseudo-random numbers.   For example, NNWT and PGNN classifiers need to know how the data are split in Super Learner discussed below to provide the perfect classifier for the Super Learner. 

Supposing each data split is 
equally likely, we conduct $n$ experiments and each experiment gives the error $e(L_v(F_v), V_v))$ of the learner $L_v(F_v)$ on $V_v$.  The cross-validated error is the average error across all $n$ experiments:
\begin{equation}
\label{EQ:sampleaverage}
\frac{1}{n} \sum_{v=1}^{n} e(L_v(F_v), V_v))
\end{equation}   

When we compare any two learning systems, they must be on the same set of learning conditions.
Weng \cite{WengMisconduct23} proposed four Learning  Conditions for comparing learning systems:  (1) A body including sensors and effectors, (2) a set of restrictions of learning framework, including whether task-specific or task-nonspecific, batch learning or incremental learning, and network refreshing rate; (3) a training experience that includes the fit set and the order of elements in the fit set, and (4) a limited amount of computational resources including the number of neurons in the system, the number of weights in the system, the storage size of the system, and the time, the money, and man-power used during the development of the system.

Let us further consider Learning Condition (4).   Suppose that one is allowed to use 
a finite, but unspecified, amount of resources in terms of computational resources and the time for system development.  This is often the case with a publically listed company.   Namely, a rich company can afford more resources.  Using NNWT and PGNN, we will see that this loop-hole will ``beat'' the traditional cross-validation---cross-validation on date splits. 

\subsection{Fit Error and Validation Error}

All publicly available data sets provide the test set to the trainer.  This is a big problem since it invites misconduct.  Surprisingly, in all ImageNet Contests \cite{Russakovsky15} the test set was publically released (in the form of an unlimited number of tests available from the test server of the contest organizer) long before the contest results were due.   Therefore, the organizers seem to have mismanaged the ImageNet Contests.

Let $\a\in A$ be a hyperparameter vector of the architecture (e.g., including receptive fields of neurons).  
Let $\w\in W$ be a weight vector of neural networks.
A neural network with $\a$ and $\w$ is denoted as $N(\a, \w)$.  The fit error of $N(\a, \w)$ is the neural network's 
average of output errors using all the input-output pairs in $F$: $e_f(\a, \w) = g_e (N(\a, \w) |F)$, where
\begin{equation}
\label{EQ:expect}
g_e (N(\a, \w) | F) = \displaystyle \mathop{\mathbb{E}}_{(\x,\z)\in F} ( \mbox{abs error of $\z$ from $N(\a, \w)$}  |F)
\end{equation}
where $\E_{\x\in X}$ is the expectation operator across all $\x \in X$.  
Note the difference between Eq.~\eqref{EQ:sampleaverage} and Eq.~\eqref{EQ:expect}.  The former is for a limited number of samples but the latter is for probablity. 
The validation error of $N(\a, \w)$ is the neural network's 
average of output errors using all the input-output pairs in 
\begin{equation}
\label{EQ:expectV}
e_{\rm v}(\a, \w) = g_e (N(\a, \w) | V).
\end{equation}
We use a practical number $n$ in Eq.~\eqref{EQ:sampleaverage} to approximate Eq.~\eqref{EQ:expectV}.
With other factors unchanged, the larger $n$, the better the approximation.

\subsection{Search Architectures but Randomly Sample Weights}

Suppose we have a hyperparameter vector $\theta \in \Theta$ and random variable $e$.   The probability density function $f_{\theta}(e)$ deterministically depends on $\theta$.  Recalling the definition of probability 
density function $f_{\theta}$, the probability for the observed value of $e$ to be less than $x$ is called the cumulative distribution function: 
\[
F_{\theta} (x) = P (e < x ) = \int_{-\infty}^{x}  f_{\theta}(e)de
\]

In Deep Learning, there are two kinds of parameters in $\theta$, $\theta = (\a, \w)$ where $\a$ is the hyperparameter vector for the architecture and $\w$ is the weight vector of the network. 

Since the dimension of  $\w$ is extremely high (e.g., 60-million dimensional in \cite{Krizhevsky17}),  
Deep Learning samples multiple initial weights $\w_i^{(0)}$, $ i=1, 2, ... , n$ (e.g., $n=20$ in \cite{Graves16} and $n=10,000$ in \cite{Saggio21}) using a pseudo-random seed, and then a greedy gradient-based search reaches a local minima weight $\w_i$, $ i=1, 2, ... , n$.   

How should we deal with $n$ networks?   The Deep Learning misconduct reports the luckiest network that corresponds to the lowest error (on $V$ or $V\cup T$) but hides all other less lucky networks.  Why is this a misconduct? 

\subsection{Report Random Distributions}
Suppose that $\theta$ is either random or deterministic.  Deep Learning trains $n\ge 2$ networks and measures the error 
$e(\theta_i)$, $i=1, 2, ... n$, of the network on $V$ with parameter vector $\theta_i$.  Regardless $\theta$ is random or deterministic initially, $e(\theta_i)$ is considered random, similar to a biological system. 

Randomly sampling $\theta$: Because we do not know the distribution of the  
the probability density $f (\x, \z, \theta)$ of the machine learner $f$.  Worse, $f$ depends on a random system parameter (vector) $\theta$.  We randomly sample the space of $ \theta$ as  $\theta_i$, and measure error $e(\theta_i)$, $i=1, 2, ... , n$.  We would like to give the best estimate $e$ that minimizes the  mean square error (MSE)
\begin{equation}
\label{EQ:min}
e^* = \argmin_{e\in R} g(e) = \argmin_{e\in R} \sum_{i=1}^n (e- e(\theta_i))^2 P_i 
\end{equation}
where $P_i$ is the probability of $\theta _i$.  It is easy to prove below that the above best solution $e^*$ for $e$ is
\begin{equation}
\label{EQ:e*}
e^* = \sum_{i=1}^n e(\theta_i) P_i 
\end{equation}
In the absence of further information, we assume $P_i=1/n$, the above expression gives
\begin{equation}
\label{EQ:e*n}
e^* = \bar{e} = \frac{1}{n}  \sum_{i=1}^n e(\theta_i)  
\end{equation}
From the above derivation, we have the following theorem:
\begin{theorem}[general cross-validation]
\label{TM:CV}
The minimum MSE estimate of a random variable $e$ from $n$ random samples, $e(\theta_i)$, $i=1, 2, ... , n$, is its probability mean Eq.~\eqref{EQ:e*}.  Thus, the general cross-validation should use  Eq.~\eqref{EQ:e*n} if we assume each sample is equally likely. 
\end{theorem}
\begin{proof}  Taking the derivative of the term, denoted as  $g(e)$, under minimization in Eq.~\eqref{EQ:min} with respect to 
$e$, and then setting it to zero, we get
\[
0 = \frac{d g(e) }
{d e} = \sum_{i=1}^n (e- e(\theta_i)) P_i
\]
Using $\sum_{i=1}^n P_i=1$, we simplify the above and get Eq.~\eqref{EQ:e*}.
The above reasoning gives the remaining proof.
\end{proof}

Although it is for general cross-validation, Theorem~\ref{TM:CV} is also the optimality of the traditional cross-validation for data splits  \cite{McLachlan92,Cherkassky98,DudaHartStork}---reporting the average of $n$ data splits. 

Theorem~\ref{TM:CV} has also established that Post-Selection for the luckiest network on $V$ is statistically invalid.  Instead of reporting the luckiest error on $V$, we must report the average errors of all $n$ trained networks on $V$.
Otherwise, the reported luckiest error inflates the expected performance in a new future test on $T$. 

The sample standard deviation of $\bar{e}$ in Eq.~\eqref{EQ:e*n} should also be reported: 
\begin{equation}
\label{EQ:std}
\hat\sigma =  \sqrt{ \frac{1}{n-1} \sum_{i=1}^n (e(\theta_i)-\bar{e}) ^2}.
\end{equation}

However, the sample standard deviation is often not informative.  It is better to report the 5 percentage locations,  0\%, 25\%, 50\%, 75\%, and 100\%, of the ranked errors on $V$.

\subsection{Statistical Flaw of Post-Selection}

Suppose one ranks the errors on $V$, 
\[
e(\theta_1) \le e(\theta_2) \le ... \le e(\theta_n). 
\]
 Since $e(\theta_i)$'s are errors, he might fraudulently report the luckiest sample $e(\theta_1)$ (like reporting the luckiest lottery ticket), but $e(\theta_1)$ badly under-estimates $e=\E(e(\theta)| T)$ (where $T$ is like a new lottery!).  The smallest sample $e(\theta_1)$ is only the luckiest on $V$ but only the average in Eq.~\eqref{EQ:e*n} is the best estimator on unknown $T$ since it reaches the minimum variance $\hat\sigma$ on $T$.   Recalling basic knowledge of probability, among $n$ random samples, 
the minimum sample badly underestimates the mean of all $n$ samples.
  
Therefore, PSUVS uses the luckiest (smallest) sample $e(\theta_1)$ to replace the average $\hat{e}$ of $n$ samples on $V$, thus, badly under-estimating $e=\E(e(\theta)| T)$ on $T$.   PSUTS is worse than PSUVS ethically since $T$  should not be leaked to the author at all!

\subsection{Hyperparameters Are Also Random}
Suppose $\theta $ is an architecture hyperparameter whose definition is known but not its value (e.g., learning rate).  Then, since the best (or hand-tuned)  $\theta$ depends on 
random data set $D$, at least its $F$ and $V$.  However, we deterministically search for $\theta$,  i.e., 
the probability density $f_{\theta} (\x, \z)$ is deterministic on parameter $\theta$, but the optimal $\theta $ 
must depend on random $F$ and $V$.   Therefore, as a function of random $F$ and $V$, 
$\theta$ is a random function that depends on $F$ and $V$.   Therefore, we should not hope
\begin{equation}
\label{EQ:transfer}
e(\theta_1, V) \stackrel{?}{\approx} \displaystyle \mathop{\mathbb{E}}_{T}  e(\theta_1, T),
\end{equation}
which is hardly possible since $\theta_1$ is the luckiest only on $V$ not necessarily the luckiest on most $T$s.  Instead, we should report the minimum variance estimator for the error from $\theta_i$ on unknown future $T$:  
\begin{equation}
\label{EQ:reportAverage}
\displaystyle \mathop{\mathbb{E}}_{T} e(\theta_i, T)  \approx \frac{1}{n} \sum_{i=1}^{n} e(\theta_i, V).
\end{equation}
$i=1, 2, ... , n$.
The expression is from Theorem~\ref{TM:CV} because the general cross-validation procedure told us that the minimum MSE estimate of any random parameter $\theta_i$ on a future test set $T$ is the sample 
average in Eq.~\eqref{EQ:e*n}.  In particular, this is also true for the luckiest $\theta_1$ on $V$, namely,
\begin{equation}
\label{EQ:1reportAverage}
\displaystyle \mathop{\mathbb{E}}_{T} e(\theta_1, T)  \approx \frac{1}{n} \sum_{i=1}^{n} e(\theta_i, V).
\end{equation}
We have derived the following theorem:
\begin{theorem}[Lost Luck]
\label{TM:LL}
The luckiest network on the validation set $V$ has the same expected error as any other less lucky networks and 
the following is the minimum MSE estimate of their errors in a future test:
\begin{equation}
\label{EQ:LL}
\displaystyle \mathop{\mathbb{E}}_{T}  e(\theta_1, T) =
\displaystyle \mathop{\mathbb{E}}_{T}  e(\theta_i, T) \approx
 \frac{1}{n} \sum_{i=1}^{n} e(\theta_i, V).
\end{equation}
$i=1, 2, ... , n$.
\end{theorem}
\begin{proof}
As we discussed above, the hyperparameter randomly depends on $F$ and $V$, and therefore, the general cross-validation principle in Theorem 
~\ref{TM:CV} applies. 
\end{proof}

In summary, there is no difference between a random $\theta_i$ or a deterministic $\theta _i$ at all in terms of our goal to give the minimum MSE of $\theta_i$ on a future test set $T$.

Summarizing the above derivation, we have the following theorem.
\begin{theorem}[Must report all trained networks]
\label{TM:reportAll}
 From each searched architecture $\a_i$, $i=1, 2, ... , k$, the optimal network weight vector is $\w_{i,j}^*$, $j=1, 2, ... , n$.  We get $kn$ networks $N(\a_i, \w_{i,j}^*)$, $i=1, 2, ... , k$, $j=1, 2, ... , n$.   Then, the minimum MSE estimate of the expected error of any network $N(\a_i, \w_{i,j}^*)$ among these $kn$ networks in a future test $T$  is the average of their errors on $V$.
 \end{theorem}
\begin{proof}
Since $\a_i$ and $\w_{i,j}^*$ all depend on the random data sets $F$ and $V$, the conclusion directly follows from the Lost Luck theorem. 
\end{proof}

In summary, the Deep Learning scheme generates $n\ge 2$ networks, not one.  This resampling method is indeed used in cross-validation \cite{McLachlan92,Cherkassky98,DudaHartStork}, whose goal is to estimate how a statistical scheme predicts on unknown future test $T$ using the average behavior of all $n$ random resamples.  The larger the $k$ value, the better the architectures are sampled, 
The larger the $n$ value, the better the weight vectors are sampled. 
 But larger $k$ and $n$ are computationally expensive.  
 
In Deep Learning, both PSUTS (many authors did) and PSUVS (few authors did) badly under-estimate the expected error on $T$ by replacing the 
average of samples in Eq.~\eqref{EQ:LL} with that of the luckiest $N(\a_i, \w_{i,j}^*)$ on $V$.   Many evolutional algorithms could contain the same misconduct because agents have $V$ but lack a future test $T$. 

Weng \cite{WengCLICCE22}  suggested further to estimate the sensitivity of the expected error of $\a$ on future test $T$  (i.e., Type-3 luck).  The lower the sensitivity, the better the stability of the performance from architecture $\a$.  In contrast, Weng \cite{WengDN3-RS22} uses a dynamic (developmental) architecture that starts from a single cell.

Weng \cite{WengCLAIEE22} reasoned that Deep Learning is like a liar in a lottery who claims that his scheme is intelligent.  The liar reports only the luckiest ticket's money return on $V$ to imply the money return on a future lottery event $T$.  This is baseless.  He must report the average of all $n$ lottery tickets on $V$.

\section{Input Cross-Validated Post-Selection}
\label{SE:CVPost}

In this section, we discuss why traditional (input) cross-validation cannot rescue Post-Selection.   Fig.~\ref{FG:Nest-Cross-Validation} is for cross-validation for both input and output.   

\begin{figure}[tb]
     \centering
     \includegraphics[width=0.48\textwidth]{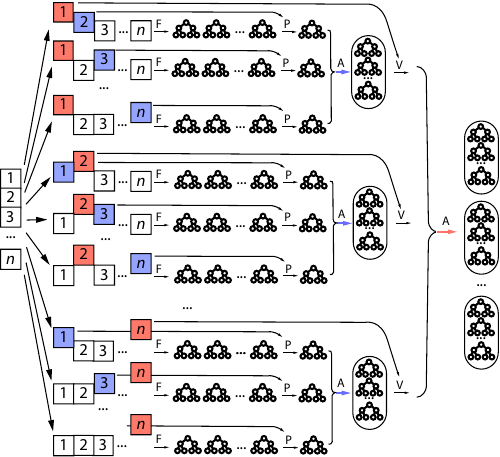}
     \caption{Nest cross-validation for post-selections.  Operator abbreviations:   F: Fit.  P: Post-Selection.  V: Validation. A: Average.  The nest cross-validation has two stages, early cross-validation (blue) before Post-Selection based on $V$, and the later cross-validation (red) after Post-Selection. 
 A blue arrowed A is due to the early cross-validation using blue data folds as validation.  The red-arrowed A is due to the latter cross-validation using the red data folds as validation.  Cross-validation results in a system that consists of all networks (or a sufficient number of representatives) that participate in the average performance.   The post-selection P here selects the luckiest single network (or few $m \ll n$) according to the error on the blue validation set.}
 \label{FG:Nest-Cross-Validation}
\end{figure}

We prove that NNWT and PGNN can reach a zero cross-validated error and they ``beat'' all machine learning systems that use Post-Selection.

For simplicity, we consider $n$-fold cross-validation.   Almost all existing papers on Deep Learning do not
have a cross-validation.   

NNWT is probably faster than PGNN to develop since NNWT tries to use the distance between a query and its nearest neighbor.  However, PGNN is simpler than NNWT by simply setting the threshold to zero. 

First, consider NNWT.  NNWT stores all samples in $F$.  Given any query $\x_q$, NNWT finds the nearest 
neighbor $\x_i$ in $F$.  If the distance between $\x_q$ and $\x_i$ is smaller than the threshold, NNWT outputs the label of $\x_i$.  Otherwise, NNWT guesses an $\x_i$ specific label for $\x_i$ when the query is 
farther than the threshold. 

\begin{theorem}[Input Cross-Validated Post-Selection]
Suppose that the fit set and validation set are both in the possession of the author.  Then, NNWT and
PGNN with input cross-validation can give a zero validation error. 
\end{theorem}
\begin{proof}
The traditional Cross-Validation gives $n$ networks, not one.   The error from the cross-validation is the average of $n$ networks.  Therefore, NNWT also gives $n$ networks.  For each 
$F_i = D-V_i$, NNWT gives one network $N_i$ that perfectly fits $F_i$.   For each sample in 
$V_i$, the network $N_i$ guesses a label if the sample is far.  As proven in \cite{WengMisleadAIEE23}, $N_i$ 
gives a zero error on $V_i$, $i=1, 2, .., n$. The proof for PGNN is similar.  As proven in \cite{WengMisconduct23}, PGNN also gives a zero error on $V_i$, $i=1, 2, .., n$. 
\end{proof}

There are at least two problems with NNWT and PGNN.  

First, they may take a very long time the get $n$ networks ready, meaning a very long developmental time that is not taken into account by Post-Selection.  

Second, they generalize poorly.   NNWT assumes that the label of a query $\x_q$ is a function of its distance to the nearest neighbor in $F_i$, which is often not true in real-world problems.  PGNN assumes that the label of a query $\x_q$ is a constant label of its nearest neighbor in $F_i$, which is often also not true in real-world problems.  

\section{Nest-Cross-Validated Post-Selection}
\label{SE:CVNPost}

The bested cross-validation performs cross-validation for both ends, the input end data splits and 
the output end data splits, as shown in Fig.~\ref{FG:Nest-Cross-Validation}.  The new idea of 
nest cross-validation is to assign a data fold not only for input but also a data fold for output. 

Let us recall that Super Learner \cite{Laan07} does not propose this nest cross-validation. For example, Fig. 1 in \cite{Laan07} contains only cross-validation for input, but not for output.  The paper text does not mention cross-validation for output either. 

\begin{theorem}[Nest-Cross-Validated Post-Selection]
Suppose that the fit set and validation set are both in the possession of the author of a paper, but not the test set.  Then, NNWT and
PGNN with input-output cross-validation can give a zero validation error. 
\end{theorem}
\begin{proof}
The traditional input-output Cross-Validation gives $kn$ networks, not one, where $n$ is the fold for 
input cross-validation and $k$ for output cross-validation.   The error from the cross-validation is the average of $kn$ networks.  Therefore, NNWT also gives $kn$ networks.  For each 
$F_{ij} = D-V_{ij}$, $i=1, 2, ... , n$ for inputs, $j=1, 2, ..., k$ for outputs, NNWT gives one network $N_{ij}$ that perfectly fits $F_{ij} $.  For each sample in 
$V_{ij}$, the network $N_{ij}$ gueses a label if the sample is far.  As proven in \cite{WengMisleadAIEE23}, $N_{ij}$ 
gives a zero error on $V_{ij}$, $i=1, 2, .., n$, $j=1, 2, ... , k$. The proof for PGNN is similar.  As proven in \cite{WengMisconduct23}, PGNN also gives a zero error on $V_{ij}$, $i=1, 2, .., n$, $j=1, 2, ... , k$. 
\end{proof}
 
 As we discussed in the above section, NNWT, and PGNN, the above problems also exist for nest cross-validation that use Post-Selections.   Therefore, both NNWT and PGNN are impractical.  They give 
 misleading low validation errors for both input and output.  Consequently, the Super Learner \cite{Laan07}
 would give 100\% weight to them and drop all other candidate classifiers, but these two classifiers are 
 impractical.   
 
 The superficially low validation error from them, although cross-validated at both ends, 
 is not transferrable to a future test.   Tables 2, 3, and 4 in \cite{Laan07} indicate that the weighted sum from
 the Super Learner fits the validation set $V$ better, but these tables lack a test that corresponds to the ``future'' test $T$ here.
 
 Laan et al. \cite{Laan07} wrote, ``For the 4 studies (synthetic data), the learning sample contained 200 observations and the evaluation sample contained 5000 observations.  However, they seem to talk about 
 $V$ that contains 5000 observations, not a new test after their Post-Selection.  The author assumes that Laan et al. \cite{Laan07} lacks a test.  If they did, the test error would still be about the average over all the candidates.

\section{Social Issues}
\label{SE:social}
Social issues are always highly complex since they must deal with intelligence from human groups.  The perspectives provided above seem to be useful for social issues. 

\subsection{Post-Selection is Invalid}

Is Post-Selection valid?   The work here has proven that Post-Selection is invalid statistically even in the presence of nest cross-validation.   Post-selection deals with only a single random sample in the space of $F$ and $V$.  The single sample represents a single generation, a limited scope in geometry, and a limited span in history.  

It seems to be addictive for a human group to simply Post-Select a model from a limited historical context (a
scientific experiment) without considering that the Post-Selection step is contaminated by the
bias from the human group.  The group hopes to hand-pick a random sample with minimum ``error'', but what it should do is compute the average of all random samples. 

\subsection{Resources of Development}

Another limitation of Post-Selection is a lack of consideration of the resources used in the development of a model, such as NNWT and PGNN, before its lucky learner enters the Super Learner.  The resources include storage space, the amount of computation of a model, as well as the time and manpower that have been spent to come up with a lucky model before the model enters the Post-Selection step that the Super Learner represents. 

In developmental psychology, this corresponds to the cost of cognitive development in a lifetime.   The corresponding question for a more general question---human development---is: How much does it cost to raise a child?   The more expensive a policy is, the less practical the policy is.  The more variation the resulting children are, the less reliable the policy is.  Within a large population of citizens, the luckiest child does not count much in the minimum MSE.  We should examine the average and the five percentage positions of the ranked performances. 

The same issue also exists in the development of nations.  Let us consider some examples.

\paragraph{US development} 
If the U.S. must be developed by causing many atrocities in other nations, such as the Vietnam War and the Russian-Urkrane war, then the policy of U.S. development is not very desirable overall.  Namely, post-selection of the U.S. nation among many nations by hiding other ``bad-looking'' nations like Ukraine, Afghanistan, and North Korea, is not a scientifically valid methodology.  In worldwide human development, how do we evaluate how well humans are developing?

\paragraph{USSR development} Suppose that the USSR developed a small set of national pride, such as Communist ideology, planned economy, and arms race.  Post-selection of such a national pride policy would hide other less proud aspects of the USSR, such as a shortage of daily goods and low average income (not 
the superficial per capita GDP).   Did the isolation of the USSR by the West Bloc hinder or slow down the internal reform in the USSR?  

\paragraph{Ukraine and Palestinian development}   Among several strategies, from violent wars against Russia and Israel, respectively, to compromising with these ``enemies'', Post-Selection of national pride at the expense of other human developmental metrics, seems to be simple-minded. 

This paper does not directly address national development {\em per se}.  Our statistical framework seems to also benefit national development.   Namely, we should not disregard the resource of national development, and only consider a post-selected, the luckiest system on a validation set and consequently hide all other less lucky systems.  Short-sighted human behaviors are common.  A politician would say, ``Regardless of what cost we will pay, we must ... ''.

\section{Conclusions}
\label{SE:conclusions}

This paper gives new and stronger theoretical results than the first three papers on Deep Learning misconduct \cite{WengPSUTS21,WengMisleadAIEE23,WengMisconduct23}.  (A) Generalized cross-validation should be applied to not only random weights but also searched or manually tuned
architecture hyperparameters.  This means that all Post-Selection methods do much worse than their authors reported.   (B) Cross-validation for data splits, at both the input end and the output end does not rescue Post-Selection.  The performance of a future test is still the average, much worse than the authors reported.   In other words, traditional cross-validation does not rescue Post-Selections from the misconduct.  (C) The general cross-validation principle might have some applications in social sciences, such as national development and human development.  Post-selection of certain national pride based on a validation set, such as making enemies, violence, and ideologies, appears not optimal for a future test.

\section*{Achnowlegements}
The author would like to thank Drs. Hongxiang (David) Qiu, Mark J. van der Laan, Eric. C. Polley, and Alan E. Hubbard on discussions of Super Learner. 

\bibliographystyle{plain}
\bibliography{shoslifref}

\begin{thebibliography}{10}

\bibitem{Alshammari21}
A.~Alshammari, R.~Almalki, and R.~Alshammari.
\newblock Developing a predictive model of predicting appointment no-show by
  using machine learning algorithms.
\newblock {\em Journal of Advances in Information Technology}, 12(3):234--239,
  2021.

\bibitem{Berk13}
R.~Berk, L.~Brown, A.~Buja, K.~Zhang, and L.~Zhao.
\newblock Valid post-selection inference.
\newblock {\em The Annals of Statistics}, 41(2):802--837, 2013.

\bibitem{Cherkassky98}
V.~Cherkassky and F.~Mulier.
\newblock {\em Learning from Data}.
\newblock Wiley, New York, 1998.

\bibitem{DudaHartStork}
R.~O. Duda, P.~E. Hart, and D.~G. Stork.
\newblock {\em Pattern Classification}.
\newblock Wiley, New York, 2nd edition, 2001.

\bibitem{Fukushima80}
K.~Fukushima.
\newblock Neocognitron: {A} self-organizing neural network model for a
  mechanism of pattern recognition unaffected by shift in position.
\newblock {\em Biological Cybernetics}, 36:193--202, 1980.

\bibitem{Fukushima83}
K.~Fukushima, S.~Miyake, and T.~Ito.
\newblock Neocognitron: {A} neural network model for a mechanism of visual
  pattern recognition.
\newblock {\em {IEEE Trans. Systems, Man and Cybernetics}}, 13(5):826--834,
  1983.

\bibitem{GaoBEAN21}
Q.~Gao, G.~A. Ascoli, and L.~Zhao.
\newblock {BEAN}: Interpretable and efficient learning with
  biologically-enhanced artificial neuronal assembly regularization.
\newblock {\em Front. Neurorobot}, 15:1--13, June 1 2021.
\newblock \url{https://doi.org/10.3389/fnbot.2021.567482}.

\bibitem{Graves16}
A.~Graves, G.~Wayne, M.~Reynolds, D.~Hassabis, et~al.
\newblock Hybrid computing using a neural network with dynamic external memory.
\newblock {\em Nature}, 538:471--476, 2016.

\bibitem{Hochreiter97}
S.~Hochreiter and J.~Schmidhuber.
\newblock Long short-term memory.
\newblock {\em Neural Computation}, 9(8):1735--1780, 1997.

\bibitem{HuangGB06}
G.~B. Huang and C.~K.Siew.
\newblock Universal approximation using incremental constructive feedforward
  networks with random hidden nodes.
\newblock {\em IEEE Transactions on Neural Networks}, 17(4):879--892, 2006.

\bibitem{GBHuang05}
G.~B. Huang, K.Z. Mao, C.~K. Siew, and D.~S. Huang.
\newblock Fast modular network implementation for support vector machines.
\newblock {\em IEEE Transactions on Neural Networks}, 16(6):1651-- 1663, 2005.

\bibitem{Krizhevsky17}
A.~Krizhevsky, I.~Sutskever, and G.~E. Hinton.
\newblock Imagenet classification with deep convolutional neural networks.
\newblock {\em Communications of the {ACM}}, 60(6):84--90, 2017.

\bibitem{Kusonkhum22}
W.~Kusonkhum, K.~Srinavin, N.~Leungbootnak, P.~Aksorn, and T.~Chaitongrat.
\newblock Government construction project budget prediction using machine
  learning.
\newblock {\em Journal of Advances in Information Technology}, 13(1):29--35,
  2022.

\bibitem{Mankowitz23}
D.~J. Mankowitz, D.~J. Michi, D.~Silver, et~al.
\newblock Faster sorting algorithms discovered using deep reinforcement
  learning.
\newblock {\em Nature}, 618:257--263, 2023.

\bibitem{McLachlan92}
G.~J. McLachlan.
\newblock {\em Discriminant Analysis and Statistical Pattern Recognition}.
\newblock Wiley, New York, 1992.

\bibitem{Russakovsky15}
O.~Russakovsky, J.~Deng, L.~Fei-Fei, et~al.
\newblock {ImageNet} large scale visual recognition challenge.
\newblock {\em Int'l Journal of Computer Vision}, 115:211--252, 2015.

\bibitem{Saggio21}
V.~Saggio, B.~E. Asenbeck, P.~Walther, et~al.
\newblock Experimental quantum speed-up in reinforcement learning agents.
\newblock {\em Nature}, 591(7849):229--233, March 11 2021.

\bibitem{Schrittwieser20}
J.~Schrittwieser, I.~Antonoglou, D.~Silver, et~al.
\newblock Mastering {Atari}, {Go}, chess and shogi by planning with a learned
  model.
\newblock {\em Science}, 588(7839):604--609, 2020.

\bibitem{Senior20}
A.~W. Senior, R.~Evans, D.~Hassabis, et~al.
\newblock Improved protein structure prediction using potentials from deep
  learning.
\newblock {\em Nature}, 577:706--710, 2020.

\bibitem{Serre07}
T.~Serre, L.~Wolf, S.~Bileschi, M.~Riesenhuber, and T.~Poggio.
\newblock Robust object recognition with cortex-like mechanisms.
\newblock {\em {IEEE Trans. Pattern Analysis and Machine Intelligence}},
  29(3):411--426, 2007.

\bibitem{Silver16}
D.~Silver, A.~Huang, D.~Hassabis, et~al.
\newblock Mastering the game of go with deep neural networks and tree search.
\newblock {\em Nature}, 529:484--489, January 27 2016.

\bibitem{Silver18}
D.~Silver, T.~Hubert, D.~Hassabis, et~al.
\newblock A general reinforcement learning algorithm that masters chess, shogi,
  and go through self-play.
\newblock {\em Science}, 362(6419):1140--1144, 2018.

\bibitem{Silver17}
D.~Silver, J.~Schrittwieser, D.~Hassabis, et~al.
\newblock Mastering the game of go without human knowledge.
\newblock {\em Nature}, 550:354--359, 2017.

\bibitem{Slonim21}
N.~Slonim, Y.~Bilu, C.~Alzate, R.~Aharonov, et~al.
\newblock An autonomous debating system.
\newblock {\em Nature}, 591(7850):379--384, March 18 2021.

\bibitem{Tuncal20}
K.~Tuncal, B.~Sekeroglu, and C.~Ozkan.
\newblock Lung cancer incidence prediction using machine learning algorithms.
\newblock {\em Journal of Advances in Information Technology,}, 11(2):91--96,
  2020.

\bibitem{Laan03}
M.~J. {van der Laan} and S.~Dudoit.
\newblock Unified cross-validation methodology for selection among estimators
  and a general cross-validated adaptive epsilon-net estimator: Finite sample
  oracle inequalities and examples.
\newblock Technical Report Paper 130, Division of Biostatistics, University of
  California, Berkeley, Berkely, CA, 2003.

\bibitem{Laan06}
M.~J. {van der Laan}, S.~Dudoit, and A.~W. {van der Vaart}.
\newblock The cross-validated adaptive epsilon-net estimator.
\newblock {\em Statistics and Decisions}, 24(3):373--395, 2006.

\bibitem{Laan07}
M.~J. {van der Laan}, Eric~C. Polley, and A.~E. {Hubbard}.
\newblock Super learner.
\newblock {\em Statistical Applications in Genetics and Molecular Biology},
  6(1):Article 25, 2007.

\bibitem{WengNatureReport21}
J.~Weng.
\newblock Data deletions in {AI} papers in {Nature} since 2015 and the
  appropriate protocol.
\newblock
  \url{http://www.cse.msu.edu/~weng/research/2021-06-28-Report-to-Nature-specific-PSUTS.pdf},
  2021.
\newblock submitted to {\em Nature}, {June} 28, 2021.

\bibitem{WengScienceReport21}
J.~Weng.
\newblock Data deletions in {AI} papers in {Science} since 2015 and the
  appropriate protocol.
\newblock
  \url{http://www.cse.msu.edu/~weng/research/2021-12-13-Report-to-Science-specific-PSUTS.pdf},
  2021.
\newblock submitted to {\em Science}, Dec. 13, 2021.

\bibitem{WengPSUTS-ICDL21}
J.~Weng.
\newblock A developmental method that computes optimal networks without
  post-selections.
\newblock In {\em Proc. IEEE Int'l Conference on Development and Learning},
  pages 1--6, Beijing, China, August 23-26 2021. NJ: IEEE Press.

\bibitem{WengPSUTS21}
J.~Weng.
\newblock On post selections using test sets {(PSUTS)} in {AI}.
\newblock In {\em Proc. Int'l Joint Conference on Neural Networks}, pages 1--8,
  Shenzhen, China, July 18-22 2021. NJ: IEEE Press.

\bibitem{WengCLICCE22}
J.~Weng.
\newblock {3D-to-2D-to-3D} conscious learning.
\newblock In {\em Proc. IEEE 40th Int'l Conference on Consumer Electronics},
  pages 1--6, Las Vegas, NV, USA, Jan. 7-9 2022. NJ: IEEE Press.

\bibitem{WengCLAIEE22}
J.~Weng.
\newblock An algorithmic theory of conscious learning.
\newblock In {\em 2022 3rd Int'l Conf. on Artificial Intelligence in
  Electronics Engineering}, pages 1--10, Bangkok, Thailand, Jan. 11-13 2022.
  NY: ACM Press.

\bibitem{WengDN3-RS22}
J.~Weng.
\newblock A developmental network model of conscious learning in biological
  brains.
\newblock Research Square, June 7 2022.
\newblock doi: \url{https://doi.org/10.21203/rs.3.rs-1700782/v2}.

\bibitem{WengMisconduct23}
J.~Weng.
\newblock On ``deep learning'' misconduct.
\newblock In {\em Proc. 2022 3rd International Symposium on Automation,
  Information and Computing (ISAIC 2022)}, pages 1--8, Beijing, China, Dec.
  9-11 2022. SciTePress.
\newblock arXiv:2211.16350.

\bibitem{WengMisleadAIEE23}
J.~Weng.
\newblock Why deep learning's performance data are misleading.
\newblock In {\em 2023 4th Int'l Conf. on Artificial Intelligence in
  Electronics Engineering}, pages 1--10, Haikou, China, Jan. 6-8 2023. NY: ACM
  Press.
\newblock arXiv:2208.11228.

\end{thebibliography}

\end{document}